\documentclass{article}

\PassOptionsToPackage{numbers,compress}{natbib}
\usepackage[dblblindworkshop,final]{neurips_2026}

\usepackage[utf8]{inputenc} 
\usepackage[T1]{fontenc}    
\usepackage[hidelinks]{hyperref} 
\usepackage{url}            
\usepackage{booktabs}       
\usepackage{amsfonts}       
\usepackage{nicefrac}       
\usepackage{microtype}      
\usepackage{xcolor}         
\usepackage{amsmath}
\usepackage{amssymb}
\usepackage{amsthm}
\usepackage{graphicx}
\usepackage{fontawesome5}

\newtheorem{proposition}{Proposition}

\DeclareMathOperator{\sg}{sg}
\DeclareMathOperator{\tr}{tr}
\DeclareMathOperator{\diam}{diam}
\DeclareMathOperator{\rank}{rank}
\newcommand{\Sp}{\mathcal{S}}
\newcommand{\Mol}{\mathcal{M}}
\newcommand{\Cond}{\mathcal{A}}
\newcommand{\Cand}{\mathcal{K}}
\newcommand{\R}{\mathbb{R}}
\newcommand{\E}{\mathbb{E}}

\workshoptitle{Representations for the Physical Sciences Workshop}
\title{Separability Is Not Alignability: \\Representations for Spectrum-to-Molecule Retrieval}
\author{
  \begin{tabular}{@{}c@{\hspace{1.5cm}}c@{}}
    \textbf{Karol Musielinski}$^{1,2}$ &
    \textbf{Fanurs Chi-En Teh}$^{1}$ \\
    \texttt{kmusieli@cs.cmu.edu} &
    \texttt{fanurs@cmu.edu} \\[0.6em]
    \textbf{Takahiko Koyama}$^{3}$ &
    \textbf{Artur Dubrawski}$^{1}$ \\
    \texttt{takahiko.koyama@keio.jp} &
    \texttt{awd@cs.cmu.edu}
  \end{tabular}
  \\[3.0em]
  $^{1}$Auton Lab, Carnegie Mellon University \\
  $^{2}$University of Warsaw \qquad
  $^{3}$Keio University
}

\begin{document}

\maketitle

\begin{abstract}
  Identifying a small molecule from its tandem mass spectrum is limited by paired data.
  Foundation models offer a way around this scarcity: pretrain a spectrum encoder on unlabeled spectra and a molecule encoder on a large corpus of structures, then align the two frozen representations, as in vision--language and as the Platonic Representation Hypothesis suggests.
  We audit this route on MassSpecGym with spectrum encoders retrained free of test-fold contamination.
  Alignment needs a metric relating distinct molecules, but the objectives that produce spectrum encoders never see molecular structure: we prove that some predictive objectives fix molecular identity without fixing that metric.
  Moreover, we find that a candidate-based objective rewards a spectrum-independent prior over the benchmark's decoys, which we isolate.
  What ultimately addresses data scarcity is not architectural: simulated spectra surpass frozen geometry.
\end{abstract}

\begin{center}
  \faGithub\enspace \texttt{\href{https://github.com/autonlab/msms-retrieval}{autonlab/msms-retrieval}}
\end{center}

\section{Introduction}
\label{sec:intro}

Tandem mass spectrometry (MS/MS) identifies a small molecule by fragmenting it and measuring the resulting fragment ions.
\emph{Retrieval} casts identification as ranking candidate structures against a query spectrum \citep{bushuiev_massspecgym_2024}.
To overcome the small size of spectral libraries, \citet{bushuiev_self-supervised_2026} developed a spectrum foundation model, DreaMS, pretrained on millions of unlabeled spectra, which \citet{krzakala_msalign_2026} subsequently proposed to align with a molecule foundation model by learning a projection into a shared embedding space.

This rationale is borrowed from vision--language, where aligning frozen unimodal encoders rivals joint pretraining \citep{zhang_sail_2025,vouitsis_fusemix_2024}.
More broadly, it echoes the Platonic Representation Hypothesis \citep{huh_platonic_2024}.
\citet{kalia_jestr_2025} state outright that molecules and their spectra are ``views of the same data.''
These parallels motivate a closer examination of the assumptions behind transferring frozen-encoder alignment to molecule--spectrum retrieval, and of where those assumptions may break down.

\textbf{Contributions.}
We audit this pipeline on MassSpecGym with spectrum encoders retrained free of test-fold contamination.
\textbf{(i)}~Satisfying the common assumptions is difficult in practice; frozen alignment instead relies on separability and metric compatibility (\S\ref{sec:setup}).
\textbf{(ii)}~We introduce MS-JEPA, a cross-condition predictive spectrum encoder, and prove that such objectives fix molecular identity without fixing the metric (\S\ref{sec:msjepa}).
\textbf{(iii)}~Measured, the two properties come apart: MS-JEPA separates molecules most strongly, yet retrieves worst when frozen (\S\ref{sec:geometry}).
\textbf{(iv)}~The candidate objective rewards a spectrum-independent prior over the benchmark's decoys, which a spectrum-blind control isolates (\S\ref{sec:adapter}).
\textbf{(v)}~End-to-end training makes spectrum pretraining nearly irrelevant, while simulated spectra of $5.9$M molecules double a fresh encoder's MRR (\S\ref{sec:baselines}).

\section{The alignment premise}
\label{sec:setup}

Let $\Mol$ be molecular structures, $\Cond$ acquisition conditions (adduct, instrument, and collision energy), and $\Sp$ spectra.
For $m\in\Mol$ and $a\in\Cond$, a measurement is $s\sim p(\,\cdot\mid m,a)$.
Retrieval ranks a candidate set $\Cand(s)\subset\Mol$ containing the true molecule $m^\star$.
Frozen alignment fixes $f:\Sp\times\Cond\to\R^{d_s}$ and $g:\Mol\to\R^{d_m}$, learns adapters $u_\theta,v_\theta$ into $\R^d$, and assigns the score
\begin{equation}
  \rho_\theta(s,m)=\tau^{-1}\big\langle \bar u_\theta(f(s,a)),\,\bar v_\theta(g(m))\big\rangle,
  \qquad \bar x \equiv x/\lVert x\rVert_2 .
  \label{eq:score}
\end{equation}

\textbf{What is borrowed.}
The Platonic Representation Hypothesis offers one explanation for why the vision--language recipe works \citep{huh_platonic_2024}.
It treats a representation as the kernel $K_f(x,x')=\langle f(x),f(x')\rangle$ it induces and claims that the kernels of different modalities converge as models scale, with convergence measured by neighborhood agreement (mutual-$k$NN), because coordinates need not coincide.
If the kernels already agree, a \emph{low-complexity adapter} should suffice.

\textbf{What it assumes.}
The hypothesis argues from a latent world state $z$ observed through modality-specific maps $x=\mathrm{obs}_x(z)$ and $y=\mathrm{obs}_y(z)$, taken to be deterministic and bijective.
Multitask pressure, capacity, and simplicity bias then push every modality's kernel toward the kernel of $z$.
However, the authors state that \emph{lossy} or \emph{stochastic} observations lie outside the argument, and mutual-$k$NN itself becomes ambiguous under non-bijective correspondence \citep{tjandrasuwita_alignment_2025,koepke_platos_cave_2026}.
If we assume the world state is the molecule itself, then the spectrum side departs from a bijection in two ways.
First, one molecule can yield many spectra, as the measurement itself is a condition-dependent stochastic draw.
Second, some isomers give near-identical spectra \citep{crotti_tandem_2023,wu_statistical_2023}.

\textbf{What kernel agreement means here.}
To compare the kernels, we lift the molecule kernel to spectra, $K_g(s,s')=\langle g(m(s)),g(m(s'))\rangle$ with $m(s)$ the molecule that produced $s$.
Write $\Sigma_m\subset\R^{d_s}$ for the support of normalized spectrum embeddings $f(s,a)$ over $s\sim p(\,\cdot\mid m,a)$.
For unit-normalized features, kernel error at most $\varepsilon$ forces $\diam(\Sigma_m)\le\sqrt{2\varepsilon}$, and exact agreement collapses every $\Sigma_m$ to a point (Prop.~\ref{prop:raw_collapse}; App.~\ref{app:proof}).
This is a stringent requirement for a spectrum encoder trained on variable, condition-dependent measurements.
Indeed, across audited encoders, the mean squared distance between same-molecule embeddings remains $0.14$--$0.51$ (Table~\ref{tab:geometry-full}; App.~\ref{app:probes}).

\textbf{Two ingredients.}
Such collapse is not necessary for retrieval: the correct molecule need only score above the other candidates after adaptation.
Successful frozen alignment with limited paired data depends on two properties: \emph{separability}, meaning that spectrum embeddings distinguish molecular identities, and \emph{metric compatibility}, meaning that their between-molecule geometry is already similar to that of the molecule tower.
Clustering spectra by molecular identity may facilitate alignment: in vision--language, \citet{zhang_sail_2025} find that clustering quality strongly correlates with retrieval performance after lightweight alignment.
Yet separability alone does not constrain the arrangement of different molecules and therefore cannot guarantee metric compatibility.

\section{MS-JEPA: learning separable spectrum representations}
\label{sec:msjepa}

To investigate whether stronger separability supports better retrieval performance, we train an encoder to capture molecular information shared across measurements under different acquisition conditions.

\textbf{Objective.}
Molecular identity is invariant across acquisition conditions, and under fixed conditions a molecule's fragmentation pattern is highly reproducible \citep{lam2011building,wang_neural_2026,bremer2022predict}.
We therefore predict a spectrum's representation from replicate spectra of the same molecule and the target condition \citep{baevski_data2vec_2022}.
With online encoder $f_\theta$, a target encoder $f_\xi$ whose weights are an exponential moving average of $\theta$ \citep{he_momentum_2020}, a stop-gradient $\sg$, a context aggregator \(\mathrm{Agg}\), and a condition-aware predictor $h_\psi$,
\begin{equation}
  \mathcal{L}_{\mathrm{JEPA}}
  =\E_{m}\E_{(C,t)}\Big[1-\cos\!\big(
    h_\psi(\mathrm{Agg}_{i \in C}(f_\theta(s_i,a_i), a_i), a_t),
    \;\sg[f_\xi(s_t,a_t)]
  \big)\Big].
  \label{eq:jepa}
\end{equation}
\citet{ghaemi_seq-jepa_2026} showed that this family of models learns invariant and equivariant representations from sequences of views and actions.
Here the action is the change of acquisition condition.

\textbf{Architecture.}
$f_\theta$ is a transformer over peak tokens with FiLM conditioning \citep{perez_film_2017}: with $e(a)$ an embedding of the ion condition (adduct, charge, polarity), block $\ell$ scales its attention and MLP residual updates by gains $1+\gamma_\ell(e(a))$.
The representation is a learned \([\mathrm{CLS}]\) token.
The predictor $h_\psi$ is a shallow MLP whose hidden activations are FiLM-modulated by the full target condition.
Each training episode samples one molecule, its target spectrum, and up to three context spectra.

The EMA teacher provides slowly evolving prediction targets to stabilize training.
In principle, the aggregator combines the representations of multiple context spectra into a molecule-level representation, then $h_\psi$ predicts the condition-dependent features.
We use a cosine loss because it cannot be reduced by simply shrinking embedding norms.
Training details are in App.~\ref{app:methods:msjepa}.

\textbf{Limitations.}
The objective compares spectra of the same molecule only, without explicitly constraining the arrangement of different molecules.
In particular, it provides no direct incentive for their geometry to agree with that of the molecule tower.

\begin{proposition}[Identity without metric]
  \label{prop:relabel}
  Let $L(f)=\inf_h L(f,h)$ over measurable heads $h$, where $L(f,h)$ is the expectation of a function of the cosine between the head's prediction from a same-molecule context $\{(f(s_i,a_i),a_i)\}_{i\in C}$ with target condition $a_t$ and the target $f(s_t,a_t)$.

  For a finite or countable molecule family, if the supports $\{\Sigma_m\}$ are pairwise disjoint and $R$ acts on each $\Sigma_m$ as an orthogonal map $Q_m$ with the images again pairwise disjoint, then $L(R\circ f)=L(f)$.

  Moreover, if molecular identity is determined by $(s,a)$, then for any molecule tower with distinct unit-normalized vectors in dimension $d\ge 2$ and candidate lists determined by the true molecule, each containing at least one distractor, some such pair $f,R$ ranks the true molecule first for every query under direct cosine retrieval with $f$ and never first with $R\circ f$.
\end{proposition}

\noindent
Proof in App.~\ref{app:proof}.
The proposition idealizes the objective~\eqref{eq:jepa} with \(f = f_\theta = f_\xi\).
It shows that, even under perfect separability, any preference for a geometry compatible with $g$ must arise from additional inductive biases.
DreaMS's contrastive loss falls outside this class, but likewise receives no direct supervision from $g$.

\section{Against separability as a sufficient condition}
\label{sec:geometry}

To test the expectations outlined in \S\ref{sec:setup}, we measure both separability and cross-modal agreement on the challenging MCES split of MassSpecGym \citep{bushuiev_massspecgym_2024}, which tests generalization to structurally distinct molecules.
We compare three frozen encoders: DreaMS after self-supervised pretraining (masked-peak prediction), its version after contrastive fine-tuning, and MS-JEPA, whose $[\mathrm{CLS}]$ token and aggregated representation we probe separately.
We reproduce both DreaMS stages to avoid test-spectra contamination from the released model (App.~\ref{app:methods:dreams}).

\textbf{Evaluation framework.}
Within-modality separability is scored by the silhouette \citep{rousseeuw1987silhouettes} and by $k$NN purity, normalized by its ceiling.
Cross-modality agreement is measured by mutual-$k$NN between molecule-level spectrum centroids and a frozen molecule tower, following \citet{huh_platonic_2024} (App.~\ref{app:probes}).
Agreement among molecule encoders ranges from $0.11$ to $0.29$, providing a same-modality reference (App.~\ref{app:molenc}).
Frozen retrieval uses the adapters of \citet{krzakala_msalign_2026}, trained with in-batch negatives under one recipe shared by all models, and is evaluated by Recall@1 and mean reciprocal rank (MRR) of the true molecule (App.~\ref{app:eval}).

\textbf{DreaMS.}
The self-supervised stage separates weakly, with silhouette $+0.14$ (Fig.~\ref{fig:geometry}) and purity $0.68$, and its agreement with MoLFormer is $0.09$: above chance, but only a third of what the best-agreeing molecule encoders share.
Contrastive fine-tuning doubles the silhouette to $+0.28$ and increases purity to $0.78$, but leaves cross-modal agreement and retrieval MRR ($0.23$) essentially unchanged.

\begin{figure}[tbp]
  \centering
  \includegraphics[width=0.95\linewidth]{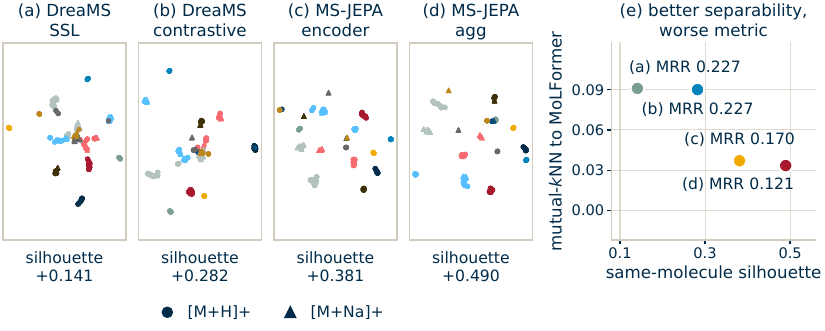}
  \caption{%
    \textbf{Separability is not alignability.}
    (a--d) UMAP of the same test spectra (12 molecules with ${\ge}12$ spectra) using four different representations.
    Color is the molecule.
    (e) Separability against cross-modal agreement with MoLFormer, annotated with frozen MRR (Table~\ref{tab:geometry-full}).%
  }
  \label{fig:geometry}
\end{figure}

\textbf{MS-JEPA.}
Despite using only molecule-grouped spectra from the MassSpecGym training fold, without large-scale pretraining, MS-JEPA achieves substantially stronger separability: silhouette $+0.49$ and normalized purity $0.91$ for the post-aggregation embedding.
Thus, cross-condition prediction is promising for learning a representation that distinguishes molecular identities, even with limited training data.
Yet MS-JEPA has the weakest agreement with MoLFormer (mutual-$k$NN $0.03$) and the worst frozen retrieval performance: $0.12$ MRR and $0.05$ Recall@1.
The strongest clustering therefore produces the weakest retrieval of the three encoders, contrary to the expectation that separability alone should facilitate frozen alignment.

\textbf{Supervising cross-modal geometry.}
We added VA-VAE's alignment loss \citep{yao_vavae_2025} to MS-JEPA's pretraining and DreaMS's contrastive stage, directly biasing the encoders towards MoLFormer's geometry.
Notably, both interventions have a positive effect on retrieval while reducing same-molecule silhouette.
MS-JEPA's frozen MRR rises to $0.19$, while mutual-$k$NN increases from $0.03$ to $0.06$.
For DreaMS, linear CKA with MoLFormer rises from $0.20$ to $0.46$, with little change in mutual-$k$NN and frozen MRR (App.~\ref{app:msjepa}).

\section{What the adapter learns}
\label{sec:adapter}

If the frozen towers do not agree, the adapters carry the performance, as the molecule foundation model makes it easy to exploit the benchmark's decoy distribution.

\textbf{The candidate objective.}
MSAlign optimizes InfoNCE \citep{oord_cpc_2018} over the candidate set $\Cand(s)$, drawn from PubChem using a mass or formula filter (App.~\ref{app:leak}, Eq.~\eqref{eq:cand}).
Writing $q$ for the induced proposal and $p$ for the library-structure marginal, the population minimizer is, up to a function of $s$ alone \citep{oord_cpc_2018,poole_variational_2019},
\begin{equation}
  \rho^\star(s,m)=\log\frac{p(m\mid s)}{q(m)}
  =\underbrace{\log\frac{p(m)}{q(m)}}_{\text{(A) prior}}
  +\underbrace{\log\frac{p(s\mid m)}{p(s)}}_{\text{(B) PMI}} .
  \label{eq:decomp}
\end{equation}
(B) is the spectrum--structure pointwise mutual information.
(A) does not depend on $s$: it is the log-odds that a structure is a library compound rather than a decoy.
With in-batch negatives \(p = q\), so (A) vanishes.
However, for conventional benchmark candidate pools \(p\) and \(q\) are substantially different \citep{gupta_confronting_2026}, raising a potential concern.

\textbf{Experiments.}
We rank every candidate by its expected score over the training spectra, $\mathbb{E}_{s}[\rho(s,m)]$, which sees neither the spectrum nor its acquisition conditions (App.~\ref{app:leak}), and measure the remaining ranking accuracy.
With in-batch negatives, performance falls to chance ($\Delta_0\le0$, Table~\ref{tab:results}, left), consistent with \eqref{eq:decomp}.
With benchmark candidates as negatives, $\Delta_0$ reaches $0.12$--$0.14$, accounting for at least half the full model's gain over random.
Candidate training increases blind MRR more than full MRR, reducing the spectral margin $\Delta_s$.
Similar but weaker effects occur in other candidate-trained models, particularly without pretrained molecule representations (App.~\ref{app:candtrained}).
This is particularly significant because the decoys do not follow the same stringent train--test splitting protocol used to support generalization claims \citep{liu_massspecgym_2026}.

\section{Improved baselines}
\label{sec:baselines}

Frozen alignment offers an efficient way to exploit foundation models trained on unpaired data, while preserving pretrained features that fine-tuning can distort under distribution shift \citep{kumar_fine-tuning_2022}.
Yet it depends on the geometry of independently trained encoders.
One alternative is to use forward spectrum predictors to generate paired data for millions of molecular structures.
Unlike spectrum-only pretraining, this directly supervises spectrum--molecule alignment and can cover molecules absent from experimental spectral libraries.

\textbf{End-to-end adaptation.}
We first unfreeze the spectrum encoder while keeping MoLFormer fixed (App.~\ref{app:methods:msmoljepa}, \ref{app:baselines}).
DreaMS improves from $0.227$ to $0.244$--$0.255$ MRR across three learning rates, and MS-JEPA from $0.121$ to $0.243$ using its aggregated representation.
Remarkably, a fresh encoder matching MS-JEPA's architecture reaches $0.232$ MRR.
Thus, despite the motivation for freezing under distribution shift, spectrum pretraining provides little advantage once paired supervision is available.

\textbf{Simulated paired data.}
We pretrain on spectra of $5.9$M leakage-audited molecules, simulated by GLACIER and FraGNNet \citep{wang_glacier_2026,young_fragnnet_2025} (App.~\ref{app:sim}), before fine-tuning on the $25$k measured training molecules.
Table~\ref{tab:results} compares global cosine scoring with fine-grained MaxSim \citep{chen_flare_2026} (App.~\ref{app:methods:msmoljepa}) under both initializations.
Simulated pretraining substantially improves global retrieval ($+0.13$--$0.14$ MRR), but its greater benefit is enabling fine-grained scoring: MaxSim adds virtually nothing on real data alone, yet gains another $0.09$--$0.10$ MRR after simulated pretraining.
The fresh encoder reaches $0.463$ MRR, within $0.013$ of pretrained MS-JEPA ($0.476$).

\textbf{Candidate-aware training.}
Adding the candidate objective raises the fresh model to $0.566$ MRR and $0.425$ Recall@1, ahead of the pretrained model ($0.542$ MRR).
However, as in \S\ref{sec:adapter}, it raises spectrum-blind MRR by $0.070$ and $0.085$.
Meanwhile, longer training on the real corpus saturates and eventually hurts validation performance (App.~\ref{app:baselines}).
Expanding molecular coverage, rather than refining the frozen representation or repeatedly fitting the same molecules, delivers the largest gains.
Fine-grained scoring benefits most once that broader supervision is available.

\begin{table}[tbp]
  \centering\scriptsize
  \renewcommand{\arraystretch}{0.85}
  \caption{%
    \textbf{Spectrum-blind ranking (left) and improved baselines (right).}
    Left: frozen DreaMS contrastive; $\Delta_0=\mathrm{MRR}_{\text{blind}}-\mathrm{MRR}_{\text{rand}}$, $\Delta_s=\mathrm{MRR}_{\text{full}}-\mathrm{MRR}_{\text{blind}}$.
    Right: Ablation of simulated pretraining, fine-grained scoring, and candidate-aware training for fresh and MS-JEPA-pretrained encoders (App.~\ref{app:baselines}).%
  }
  \label{tab:results}
  \setlength{\tabcolsep}{5pt}

  \begin{minipage}[t]{0.45\linewidth}\centering
    \begin{tabular}{llccc}
      \toprule
      Tower & Negatives & MRR & $\Delta_0$ & $\Delta_s$ \\
      \midrule
      ChemBERTa & in-batch  & 0.147 & $-0.004$ & 0.127 \\
      ChemBERTa & candidate & 0.211 & 0.122 & 0.064 \\
      MoLFormer & in-batch  & 0.227 & $-0.006$ & \textbf{0.209} \\
      MoLFormer & candidate & \textbf{0.303} & \textbf{0.144} & 0.135 \\
      \bottomrule
    \end{tabular}
  \end{minipage}\hfill
  \begin{minipage}[t]{0.53\linewidth}\centering
    \begin{tabular}{llcccc}
      \toprule
      & & \multicolumn{2}{c}{fresh} & \multicolumn{2}{c}{pretrained} \\
      \cmidrule(lr){3-4}\cmidrule(lr){5-6}
      Training & Score & MRR & R@1 & MRR & R@1 \\
      \midrule
      real & global & 0.232 & 0.124 & 0.243 & 0.128 \\
      real & MaxSim & 0.241 & 0.130 & 0.234 & 0.129 \\
      sim$\to$real & global & 0.369 & 0.229 & 0.374 & 0.238 \\
      sim$\to$real & MaxSim & 0.463 & 0.334 & 0.476 & 0.341 \\
      sim$\to$real & MaxSim + cand. & \textbf{0.566} & \textbf{0.425} & 0.542 & 0.400 \\
      \bottomrule
    \end{tabular}
  \end{minipage}
\end{table}

\section{Conclusion}
\label{sec:conclusion}

Frozen spectrum--molecule alignment borrows the premise that independently pretrained encoders share enough geometry for lightweight alignment \citep{huh_platonic_2024,zhang_sail_2025}.
Our results expose the limits of this premise for MS/MS.
Molecular separability does not determine between-molecule geometry: MS-JEPA achieves the strongest same-molecule clustering yet the weakest frozen retrieval.
Across spectrum and molecule encoders, cross-modal agreement remains low.
Concurrent work likewise reports near-chance label-free alignment between DreaMS and molecular fingerprints on MassSpecGym \citep{schnaus_rosetta_2026}.
More generally, spectral similarity is an imperfect proxy for molecular similarity: the cosine--Tanimoto correlation \citep{bushuiev_self-supervised_2026,chen_flare_2026} of self-supervised DreaMS is $0.63$ under similarity-balanced sampling but only $0.23$ under uniform pair sampling, with a substantial correlation also obtained from precursor mass alone (App.~\ref{app:correlation}).
Together, these findings reveal a broader difficulty with frozen alignment in this modality.

Retrieval performance can also obscure what the spectrum representation contributes.
Candidate-trained adapters retain substantial accuracy without query-specific spectral input, which reinforces the need for spectrum-blind controls.

Our improved baselines point to one possible alternative: rather than relying on compatible frozen representations, forward spectrum models provide paired supervision across millions of molecular structures.
End-to-end training largely eliminates the advantage of spectrum pretraining, while simulated spectra from $5.9$M molecules unlock fine-grained scoring and raise a fresh encoder from $0.232$ to $0.463$ MRR without candidate negatives.
These results identify molecular coverage as the central constraint in our setting.
Simulated paired data supplies this missing supervision at a scale that measured spectral libraries cannot currently provide.

\begin{ack}
  This work was partially supported by NSF award \#2406231 and the Artur Dubrawski Alumni Research Professorship Chair funds.
  This work used Bridges-2 at Pittsburgh Supercomputing Center \citep{brown_bridges2_2021} through allocation CHE260128 from the ACCESS program, which is supported by NSF grants \#2138259, \#2138286, \#2138307, \#2137603, and \#2138296.
\end{ack}

\bibliographystyle{rps-abbrvnat}
\bibliography{rps-bibliography}

\clearpage
\appendix

\section{Extended related work}
\label{app:related}

\textbf{Frozen-encoder alignment.}
CLIP trains both towers jointly \citep{radford_clip_2021}.
FuseMix aligns frozen unimodal encoders using latent-space augmentation \citep{vouitsis_fusemix_2024}; SAIL assesses alignment and achieves CLIP-competitive transfer with substantially less paired data \citep{zhang_sail_2025}.
MSAlign applies this approach to MS/MS \citep{krzakala_msalign_2026}.

\textbf{Convergence and its limits.}
\citet{huh_platonic_2024} report greater mutual-$k$NN agreement between more capable models.
At larger gallery sizes, \citet{koepke_platos_cave_2026} find that agreement falls from $0.135$ on $1{,}024$ samples to $0.008$ on LAION-15M ($k{=}10$); within-modality agreement remains much higher, and many-to-many correspondence further reduces cross-modal agreement.
\citet{tjandrasuwita_alignment_2025} analyze when alignment emerges.
\citet{schnaus_rosetta_2026} align unpaired embeddings by Wasserstein Procrustes, but obtain near-chance alignment between DreaMS and ECFP4 on MassSpecGym (FOSCTTM $0.47$; chance $0.5$).
Related modality-gap studies examine the relative placement of modalities in a shared space \citep{liang_mind_2022,li_closing_2025,yu_modality_2026,yu_anisotropic_2026}.

\textbf{Spectrum representations.}
Spec2Vec \citep{huber_spec2vec_2021} and MS2DeepScore \citep{huber_ms2deepscore_2021} learn spectral similarity.
MSBERT \citep{zhang_msbert_2024} and DreaMS \citep{bushuiev_self-supervised_2026} pretrain transformers on unlabeled spectra; \citet{gine_benchmarking_2026} compare featurization strategies.

\textbf{Molecule retrieval.}
Methods differ in how they compare a spectrum with candidates.
\emph{Library search} matches measured spectra using cosine, entropy \citep{li_spectral_2021}, or learned similarity \citep{huber_spec2vec_2021,huber_ms2deepscore_2021}; it requires a reference spectrum \citep{lam2011building}.
\emph{Spectrum-to-fingerprint} methods predict structural features: CSI:FingerID uses fragmentation trees and kernel SVMs \citep{duhrkop_csi_2015}, and MIST uses subformula-annotated peaks \citep{goldman_annotating_2023,heirman_reusability_2024}.
\emph{Structure-to-spectrum} methods simulate candidates for spectral matching: CFM-ID \citep{wang_cfm-id_2021,bremer2022predict}, GrAFF-MS \citep{murphy_efficiently_2023}, MassFormer \citep{young_tandem_2024}, ICEBERG \citep{goldman_generating_2024}, GLACIER \citep{wang_glacier_2026}, and FraGNNet \citep{young_fragnnet_2025}.
We use the last two for simulated pretraining.
\emph{Joint embedding} methods rank candidates in a shared space: JESTR combines a graph network and spectrum MLP with in-batch InfoNCE and a candidate-list penalty \citep{kalia_jestr_2025}; CSU-MS$^2$ uses contrastive spectrum--structure learning \citep{xie_csu-ms2_2025}; MSAlign adapts frozen encoders \citep{krzakala_msalign_2026}; and FLARE uses bidirectional peak--atom MaxSim with subformula-annotated peaks and known parent formula and adduct \citep{chen_flare_2026}.
De novo methods generate structures beyond a fixed candidate list \citep{bushuiev_massspecgym_2024,bohde_diffms_2025,wang_madgen_2025}.

\textbf{Predictive SSL.}
Equation~\eqref{eq:jepa} predicts latent targets, as in data2vec \citep{baevski_data2vec_2022}, using an EMA encoder \citep{he_momentum_2020}.
\citet{garrido_self-supervised_2023,ghaemi_seq-jepa_2026} develop invariant and equivariant representations.
Our non-identifiability result is related to, but distinct from, the impossibility of unsupervised disentanglement without inductive biases \citep{locatello_challenging_2019}.

\section{Methods}
\label{app:methods}

Table~\ref{tab:hparams} summarizes the three spectrum encoders.

\begin{table}[h]
  \centering\scriptsize
  \caption{%
    The three spectrum encoders.
    The DreaMS and MS-JEPA columns describe pretraining; the fresh encoder is trained end to end with the recipe of App.~\ref{app:eval}.
    The last two rows describe retrieval training (App.~\ref{app:methods:msmoljepa}, \ref{app:methods:msalign}).%
  }
  \label{tab:hparams}
  \setlength{\tabcolsep}{4pt}
  \begin{tabular}{lccc}
    \toprule
    & DreaMS (reproduced) & MS-JEPA & fresh encoder \\
    \midrule
    Blocks / width / heads / MLP & 7 / 1,024 / 8 / 4,096 & 8 / 768 / 12 / 3,072 & 8 / 768 / 12 / 3,072 \\
    Parameters & 95.5M & 64.1M encoder; 87.7M as (A) & 87.7M, as MS-JEPA (A) \\
    Peaks per spectrum & 60 (SSL), 100 (contrastive) & 100 & 100 \\
    $m/z$ encoding & fixed Fourier & fixed Fourier, $m/z$ and neutral loss & as MS-JEPA \\
    Attention bias & no own parameters, shared & learned, per head and block & as MS-JEPA \\
    Encoder conditioning & none & FiLM gains: adduct, charge, polarity & as MS-JEPA \\
    Representation & precursor token & $[\mathrm{CLS}]$ token, aggregated & as MS-JEPA (A) \\
    Objective & masked $m/z$, then contrastive & JEPA cosine to an EMA target & InfoNCE to MoLFormer \\
    Training data & GeMS-A10, ${\approx}24$M spectra & MSG train, 175,536 spectra & MSG train, 194,119 spectra \\
    Batch / optimizer steps & 4,096 spectra / 71,288 & 256 episodes / 10,240 & $4{\times}64$ spectra / 9,700 \\
    Optimizer & Adam $10^{-4}$, Noam & AdamW $2{\times}10^{-4}$, cosine & AdamW $10^{-4}$, cosine \\
    Precision / hardware & fp32, 8 H100 & bf16, 4 A100 & bf16, 4 A100 \\
    \midrule
    Adapter conditions & adduct, CE & none (A); CE, instrument (B) & none \\
    Trainable, frozen / trained & 8.42M / 104.0M & 7.35M / 95.0M (A) & -- / 95.0M \\
    \bottomrule
  \end{tabular}
\end{table}

\subsection{Notation}
\label{app:methods:notation}

Symbols follow \S\ref{sec:setup}.
Write $z_s=\bar u_\theta(f(s,a))$ and $z_m=\bar v_\theta(g(m))$ on $\mathbb S^{d-1}$, so $\rho_\theta(s,m)=\tau^{-1}\langle z_s,z_m\rangle$ as in \eqref{eq:score}.
$\Sigma_m$ is the support of $f(s,a)$ over measurements of $m$ under all conditions.

\subsection{The spectrum encoder of MS-JEPA}
\label{app:methods:encoder}

The input contains learned $[\mathrm{CLS}]$ and register tokens \citep{darcet_vision_2024}, the precursor (intensity $1.1$), and the $100$ most intense fragments, normalized to the base peak.
Each peak combines three branches: fixed Fourier features of $m/z$ and precursor-minus-fragment mass ($12{,}000$ features; periods $1$--$1{,}000$~Da and $0.1$--$999.9$~mDa), mapped by a shared mass MLP to $256$ dimensions each; and four intensity transforms (identity, square root, logarithm, and fraction of total fragment intensity), mapped to $128$ dimensions.
Layer normalization and a linear map fuse these branches into a $768$-dimensional token, with an added token-type embedding.

Eight pre-norm transformer blocks use $12$ heads, MLP width $3{,}072$, layer scale initialized at $0.01$ \citep{touvron2021going}, and stochastic depth up to $0.1$ \citep{huang2016deep}.
Each block learns a per-head attention bias from pairwise mass differences (Fourier features at $16$ periods from $1$~mDa to $1$~kDa), intensities, and token types.
By comparison, DreaMS's Graphormer-style bias \citep{ying_transformers_2021} uses differences of summed Fourier-branch features, shared across heads and blocks, with no additional parameters.

Adduct, charge, and polarity embeddings are projected to $128$ dimensions and modulate each block's attention and MLP residual branches through FiLM gains $1+\gamma$.
Zero initialization of the gain layer starts training with an unconditioned encoder; replacing the gains by the identity with probability $0.1$ permits use without conditions.
The trunk has $64.1$M parameters.

\subsection{MS-JEPA}
\label{app:methods:msjepa}

\textbf{Episodes.}
Spectra are grouped by molecule and condition (adduct, instrument family, collision energy).
Each episode samples a molecule with at least two condition groups, a target from one group, and context from the others: one spectrum, or three if at least four groups exist.
The training fold provides $16{,}294$ eligible molecules and $175{,}536$ spectra (adducts $[\mathrm{M{+}H}]^+$ and $[\mathrm{M{+}Na}]^+$); $5{,}510$ molecules support three contexts.

\textbf{Prediction.}
The encoder's $[\mathrm{CLS}]$ output is $z_{\mathrm{spec}}$.
A set aggregator ($22.0$M parameters) concatenates each context representation with its source condition, projects to $768$ dimensions, and applies three transformer blocks with a learned pooling token to produce $h_{\mathrm{mol}}$.
Every episode uses its first context ($K{=}1$); three-context episodes also use all three ($K{=}3$).
The condition encoder ($1.6$M parameters, mostly FiLM gains) embeds adduct, charge, polarity, instrument, and Fourier collision-energy features for the aggregator and predictor.
The predictor ($3.7$M parameters) maps $h_{\mathrm{mol}}$ through two residual MLP blocks of hidden width $1{,}024$, FiLM-modulated by the target condition.

The target is the EMA encoder's $z_{\mathrm{spec}}$ under the target spectrum's condition, with gradients stopped.
The loss averages the $1-\cos$ terms for $K{=}1$ over all episodes and $K{=}3$ over three-context episodes.
Target encoder, condition encoder, and prefix tokens follow their online counterparts by EMA, with cosine-scheduled momentum from $0.996$ to $1$.

\textbf{Optimization.}
We use AdamW ($\beta=(0.9,0.999)$), peak learning rate $2\times10^{-4}$, an $8$-epoch warm-up, and cosine decay to $10^{-6}$.
Layer-wise learning-rate decay is $0.9$, with a further $0.5\times$ factor on the peak encoder.
Weight decay rises from $0.04$ to $0.40$ on a cosine schedule; gradients are clipped at $3$.
Training uses four A100-80GB GPUs, $64$ episodes per GPU ($256$ per step), bf16, and seed $42$: $160$ epochs of $64$ steps ($10{,}240$ total), in under one hour.
Validation selects epoch $139$, step $8{,}960$, by $K{=}1$ prediction MRR among the true target and $63$ targets sampled uniformly from other validation molecules.

\textbf{Retrieval.}
We retain the online encoder, condition encoder, prefix tokens, and aggregator.
\textbf{(A)}~The encoder receives ion conditions; the aggregator processes one spectrum ($K{=}1$) with the full source condition (adduct, charge, polarity, instrument, Fourier collision-energy features, and validity bit).
The adapter receives $h_{\mathrm{mol}}$ and no conditions ($87.7$M encoder parameters).
\textbf{(B)}~The adapter receives $z_{\mathrm{spec}}$, collision energy, and instrument; the encoder sees only ion conditions ($65.7$M parameters).
Results use (A) unless stated otherwise.

\subsection{End-to-end and fresh encoders}
\label{app:methods:msmoljepa}

End-to-end models train the spectrum encoder and both adapters while keeping MoLFormer frozen.
Initialization is DreaMS contrastive, MS-JEPA (A) or (B), or a fresh MS-JEPA (A) encoder with identity FiLM gains.
The global vector ($h_{\mathrm{mol}}$, $z_{\mathrm{spec}}$, or DreaMS's precursor token), with the same adapter conditions as in the frozen model, feeds the spectrum adapter (App.~\ref{app:methods:msalign}).
No predictor or extra query/readout tokens are used.
DreaMS dropout ($0.1$), MS-JEPA stochastic depth, and FiLM condition dropout remain active; DreaMS Fourier frequencies remain fixed.
Encoder learning rates are $5\times10^{-6}$, $1.5\times10^{-5}$, or $5\times10^{-5}$ for pretrained models and $10^{-4}$ or $3\times10^{-4}$ for fresh models; adapters use $10^{-4}$.
Other settings follow App.~\ref{app:eval}.
Table~\ref{tab:results} uses validation-selected configurations (App.~\ref{app:baselines}).

\textbf{Fine-grained scoring.}
The spectrum side is the set $S(s)$ of the global vector and the encoder's last-layer precursor and fragment tokens (at most $100$), each with the adapter conditions appended and passed through the spectrum adapter.
The molecule side is the set $T(m)$ of MoLFormer's last-layer token states except BOS, EOS, PAD and MASK, each through the molecule adapter; the unknown-token state is kept, since it stands for an out-of-vocabulary atom.
With every vector $\ell_2$-normalized,
\begin{equation*}
  \rho_\theta(s,m)=\frac{1}{\tau\,|S(s)|}\sum_{u\in S(s)}\max_{v\in T(m)}\langle u,v\rangle,
\end{equation*}
the spectrum-to-molecule direction of MaxSim \citep{chen_flare_2026}.
With one vector a side it is the cosine score of \eqref{eq:score}.
Token states are computed by the frozen tower on the fly.

\subsection{MSAlign}
\label{app:methods:msalign}

Our MSAlign adapters \citep{krzakala_msalign_2026} each comprise a linear layer to $2{,}048$ dimensions, LayerNorm, GELU, dropout $0.2$, a linear layer to $1{,}024$, and $\ell_2$ normalization.
Scores are cosine similarities divided by a learned temperature initialized at $0.07$.
The spectrum input concatenates the encoder output ($1{,}024$ dimensions for DreaMS, $768$ for MS-JEPA) with $128$-dimensional embeddings of conditions absent from the encoder: adduct and collision energy for DreaMS; collision energy and instrument for MS-JEPA (B); none for (A).
Adduct and instrument use learned embeddings.
Collision energy uses sine and cosine at $16$ log-spaced frequencies ($0.01$--$10$~rad/eV), a validity bit, and a two-layer MLP.

MoLFormer-XL-both-10pct \citep{ross_large-scale_2022} and ChemBERTa-zinc-base-v1 \citep{chithrananda_chemberta_2020} are frozen at pinned snapshots in evaluation mode, with fixed random attention features for MoLFormer.
Both output $768$ dimensions.
Global vectors are MoLFormer's mean over attended tokens and ChemBERTa's $\langle s\rangle$ state.

\textbf{Training.}
We use in-batch or candidate InfoNCE (Eq.~\eqref{eq:cand}) and the recipe in App.~\ref{app:eval}; simulated pairs share identity when they share a store molecule (App.~\ref{app:sim}).
With frozen encoders, trainable parameters total $8.42$M for DreaMS ($4.73$M spectrum adapter, $3.68$M molecule adapter, $21$k condition embeddings, and temperature), $7.35$M for MS-JEPA (A), and $7.90$M for (B).

\subsection{DreaMS: deviations from the published recipe}
\label{app:methods:dreams}

We retrain both stages because the released model's contrastive stage uses MoNA, a source of MassSpecGym spectra.

\textbf{Self-supervised stage.}
We run the released checkpoint's configuration on GeMS-A10, changing hardware to eight H100-80GB GPUs (from A100-40GB), validation to $0.5\%$ of LC--MS runs (from $1\%$ of spectra), mask sampling to distribution-equivalent GPU Gumbel-top-$k$, the Noam step index by one, and checkpoint frequency to every $2{,}750$ steps.
We use step $71{,}288$.

\textbf{Contrastive stage.}
We apply the released MassSpecGym fine-tuning settings to the training fold: Adam at $3\times10^{-6}$ without scheduling, $32$ anchors per step, cosine-triplet margin $0.1$, positives sharing an InChIKey connectivity block, negatives within $0.05$~Da neutral mass, $100$ peaks, fp32, seed $3407$, and an unfrozen backbone throughout.
Anchors are restricted to $[\mathrm{M{+}H}]^+$ to prevent adduct-based separation of mass-window negatives.
We train for $25{,}000$ steps and use step $24{,}000$.

\textbf{Retrieval.}
Each checkpoint retains its training peak budget ($60$ SSL, $100$ contrastive), padded to that length because the released attention does not mask padded keys.
The representation is the final precursor token.

\section{Proofs}
\label{app:proof}

\begin{proof}[Proof of Proposition~\ref{prop:relabel}]
  Write $L(f)=\inf_h L(f,h)$, where
  \[
    L(f,h)
    =\E\,\ell\!\left(
    \left\langle
    \overline{h(X_C,A_C,a_t)},\overline{f(s_t,a_t)}
    \right\rangle\right),
    \qquad
    X_C=\{f(s_i,a_i)\}_{i\in C},
  \]
  and the context and target spectra come from the same molecule.

  Let $f'=R\circ f$ and $\Sigma'_m=Q_m\Sigma_m$.
  Since the transformed supports are pairwise disjoint, any nonempty same-molecule context uniquely determines $m$.
  For a valid transformed context $x_C\subset\Sigma'_m$, define
  \[
    h'(x_C,A_C,a_t)
    =
    Q_m h(Q_m^{-1}x_C,A_C,a_t),
  \]
  where $Q_m^{-1}x_C=\{Q_m^{-1}x_i:x_i\in x_C\}$.
  Set $h'$ to a fixed nonzero vector on other contexts.
  Since the molecule family is finite or countable and the supports are Borel, this defines a measurable predictor.

  For a context--target sample from molecule $m$,
  \[
    \overline{h'(X'_C,A_C,a_t)}
    =Q_m\overline{h(X_C,A_C,a_t)},
    \qquad
    \overline{f'(s_t,a_t)}
    =Q_m\overline{f(s_t,a_t)}.
  \]
  Orthogonality of $Q_m$ leaves the cosine unchanged, so $L(f',h')=L(f,h)$.
  Hence $L(f')\leq L(f)$.
  Applying the same construction to the inverse of $R$ on $\bigcup_m\Sigma'_m$, which acts as $Q_m^{-1}$ on each $\Sigma'_m$, gives the reverse inequality and therefore
  \[
    L(R\circ f)=L(f).
  \]

  For the second claim, let $\{\bar g(m)\}\subset\mathbb S^{d-1}$, $d\geq2$, be distinct, suppose that $m(s,a)$ is measurable, and let each $\Cand(m)$ contain $m$ and at least one other molecule.
  Set
  \[
    f(s,a)=\bar g(m(s,a)).
  \]
  Then $\Sigma_m=\{\bar g(m)\}$, and the score
  \[
    \rho_f(s,m')=\langle f(s,a),\bar g(m')\rangle
  \]
  is uniquely maximized at $m'=m(s,a)$.

  For each $m$, choose $\sigma(m)\in\Cand(m)\setminus\{m\}$ and let
  \[
    \delta_m
    =
    \angle\!\big(\bar g(m),\bar g(\sigma(m))\big)>0.
  \]
  Because $d\geq2$ and the molecule family is at most countable, we may choose pairwise distinct unit vectors $y_m$ such that
  \[
    \angle\!\big(y_m,\bar g(\sigma(m))\big)<\delta_m/3.
  \]
  Choose an orthogonal $Q_m$ satisfying $Q_m\bar g(m)=y_m$, and define $R|_{\Sigma_m}=Q_m$.
  The transformed supports $\{y_m\}$ are pairwise disjoint.
  Moreover, by the reverse triangle inequality,
  \[
    \angle\!\big(y_m,\bar g(m)\big)
    \geq
    \delta_m-\angle\!\big(y_m,\bar g(\sigma(m))\big)
    >
    2\delta_m/3
    >
    \angle\!\big(y_m,\bar g(\sigma(m))\big).
  \]
  Thus, under $f'=R\circ f$, the candidate $\sigma(m)$ has strictly larger cosine similarity than the true molecule $m$.
  Hence direct cosine retrieval with $f'$ never ranks the true molecule first, whereas retrieval with $f$ always does, despite $L(f')=L(f)$.
\end{proof}

\begin{proposition}[Raw kernel agreement contracts each support]
  \label{prop:raw_collapse}
  Let $f(s,a)\in\mathbb S^{d_s-1}$ and $g(m)\in\mathbb S^{d_m-1}$, and let $\varepsilon\ge0$.
  Suppose
  \[
    \left|
    \langle x,y\rangle-\langle g(m),g(n)\rangle
    \right|
    \le\varepsilon
  \]
  for all $m,n$ and all $x\in\Sigma_m$, $y\in\Sigma_n$.
  Then $\diam(\Sigma_m)\le\sqrt{2\varepsilon}$ for every $m$.
  In particular, if $\varepsilon=0$, every nonempty $\Sigma_m$ is a singleton.
\end{proposition}

\begin{proof}
  Fix $m$ and $x,y\in\Sigma_m$.
  Taking $n=m$ and using $\langle g(m),g(m)\rangle=1$ gives $\langle x,y\rangle\ge1-\varepsilon$.
  Because $x$ and $y$ are unit normalized,
  \[
    \|x-y\|_2^2
    =2-2\langle x,y\rangle
    \le2\varepsilon.
  \]
  Taking the supremum over $x,y\in\Sigma_m$ proves the claim.
\end{proof}

\section{Evaluation protocol}
\label{app:eval}

\textbf{Data.}
We use MassSpecGym v1.5 \citep{bushuiev_massspecgym_2024,liu_massspecgym_2026} and its official MCES split: $194{,}119$ / $19{,}429$ / $17{,}556$ train / validation / test spectra from $25{,}046$ / $3{,}386$ / $3{,}170$ query molecules ($2{,}998$ test InChIKey-14 blocks).
Each query uses its official mass-based list, including the true molecule ($12$--$256$ candidates; median $256$).
Formula-based lists are used only for the leakage audit (App.~\ref{app:sim}).
Every tower reads RDKit canonical isomeric SMILES, which also identify the true candidate \citep{liu_massspecgym_2026}.
We drop $66$ invalid candidate entries; canonicalization produces no duplicates within lists.

\textbf{Metrics.}
The rank of the true molecule is $1$ plus the number of candidates scored strictly higher (fp32 scores), so ties favor the true molecule; no run has ties on more than $4$ of the $17{,}556$ test spectra.
Mean reciprocal rank
\[
  \mathrm{MRR}
  =
  \frac{1}{N}\sum_{i=1}^{N}\frac{1}{\rank_i}
\]
and Recall@$\{1,5,20\}$ are averaged over test spectra.
Uniform random ranking gives $\E[1/\rank]=H_K/K$ for a list of $K$, MRR $0.024$ on the test fold.

\textbf{Training recipe.}
Except JESTR (App.~\ref{app:candtrained}), all retrieval models share this recipe.
Frozen encoders remain in evaluation mode with learning rate $0$.
Each epoch samples one spectrum per training query molecule in random order.
Four A100-80GB GPUs each process $64$ spectra per batch: $97$ steps per epoch, dropping $214$ molecules, for $100$ epochs ($9{,}700$ steps).

Augmentation independently drops peaks with probability $0.2$ (retaining at least one), multiplies intensities by $e^{0.3\varepsilon}$, and multiplies $m/z$ by $1+5\times10^{-6}\varepsilon$, with $\varepsilon\sim\mathcal N(0,1)$; the precursor is unchanged.
Spectra are then restricted to the encoder's peak budget and normalized to the base peak: $100$ peaks by default, $60$ for DreaMS SSL, and $128$ for the MS-JEPA ablations in App.~\ref{app:msjepa}.

AdamW ($\beta=(0.9,0.999)$) uses learning rate $10^{-4}$ for adapters and condition embeddings and the rates in App.~\ref{app:methods:msmoljepa} for encoders.
Weight decay is $0.02$ on parameters with at least two dimensions.
We use $5\%$ linear warm-up, cosine decay to zero, and gradient clipping at $1$.
Temperature is learned as $\log\tau$, initialized at $\tau=0.07$, with learning rate $10^{-2}$, no weight decay, and $\tau\ge0.01$ enforced after each step.
Forward passes use bf16 autocast; scores use fp32.
Validation on official lists runs every $200$ steps ($485$, or five epochs, for MaxSim) and at the final step; the highest MRR selects the checkpoint.
App.~\ref{app:baselines} lists deviations.

\textbf{Statistics.}
Models use one training seed ($42$), except the JESTR repeats.
Checkpoints and configurations are selected by validation MRR, then evaluated once on test; JESTR uses the last checkpoint.
We report $95\%$ percentile intervals from $1{,}000$ bootstrap samples of the $2{,}998$ test InChIKey-14 blocks, recomputing metrics over their spectra.
Model differences use paired samples.
These intervals capture test-sample uncertainty, not training variance.

\section{Geometry probes}
\label{app:probes}

Within-modality probes use the $17{,}556$ test spectra of $2{,}998$ molecules (InChIKey-14); cross-modal probes use the $2{,}994$ of them that every representation covers.

\textbf{Same-molecule silhouette.}
For each spectrum, let $a$ be its mean cosine distance to other spectra of the same molecule and $b$ its smallest mean distance to another molecule.
We average $(b-a)/\max(a,b)$ \citep{rousseeuw1987silhouettes} over spectra with at least two replicates per molecule.

\textbf{$k$NN purity.}
Purity is the same-molecule fraction among the $k$ nearest other spectra.
For a molecule with $n$ spectra, its ceiling is $\min(k,n-1)/k$ (fold mean $0.676$ at $k{=}10$).
We normalize by the per-spectrum ceiling for comparison across $k$ and corpora.
Chance is $0.0021$.

\textbf{Scatter summaries.}
$\tr S_W$ is the mean squared distance to each molecule's centroid; $\tr S_B$ is the size-weighted mean squared distance of centroids to the global mean.
Their ratio is the Fisher ratio; multiplying by $(N-C)/(C-1)$ gives Calinski--Harabasz.
Davies--Bouldin averages each molecule's largest ratio of summed within-molecule spread to between-centroid distance.
Alignment is the mean squared distance between same-molecule pairs; uniformity is the log mean Gaussian potential over all pairs.

\textbf{Cross-modal agreement.}
Mutual-$k$NN \citep{huh_platonic_2024} measures overlap between a molecule's $k$ nearest neighbors in two representations.
Spectrum embeddings are averaged and normalized per molecule.
For $n$ molecules, chance overlap is $k/(n-1)$ when the query is excluded.
Linear CKA \citep{kornblith_similarity_2019} is $\lVert Y^\top X\rVert_F^2/(\lVert X^\top X\rVert_F\lVert Y^\top Y\rVert_F)$ for column-centered matrices.
Within-modality references help distinguish poor agreement from scale effects in the estimator \citep{koepke_platos_cave_2026}.
The two DreaMS stages agree at $0.372$ but share architecture and most weights; App.~\ref{app:molenc} supplies independent molecule-encoder references.

\begin{table}[h]
  \centering\footnotesize
  \caption{%
    Frozen-geometry panel, MassSpecGym test fold.
    MS-JEPA is the aggregated representation; MS-JEPA [CLS] is the encoder's [CLS] token, which sees only the ion condition.
    Arrows give the preferred direction.
    Frozen MRR: in-batch MoLFormer adapter (App.~\ref{app:eval}); the adapter receives adduct and collision energy for DreaMS, no conditions for MS-JEPA (its encoder reads them), and collision energy and instrument for [CLS].%
  }
  \label{tab:geometry-full}
  \setlength{\tabcolsep}{3pt}
  \begin{tabular}{lcccc}
    \toprule
    Metric & DreaMS SSL & DreaMS contr. & MS-JEPA & MS-JEPA [CLS] \\
    \midrule
    \multicolumn{5}{l}{\emph{Within-modality (partition)}}\\
    Same-molecule silhouette ($\uparrow$)      & +0.141 & +0.282 & \textbf{+0.490} & +0.381 \\
    $k$NN purity@1 ($\uparrow$)                & 0.878 & 0.922 & \textbf{0.971} & 0.965 \\
    $k$NN purity@10, normalized ($\uparrow$)   & 0.683 & 0.782 & \textbf{0.905} & 0.881 \\
    Fisher trace ratio ($\uparrow$)            & 1.85 & 2.12 & \textbf{6.45} & 3.56 \\
    Calinski--Harabasz ($\uparrow$)            & 10.3 & 11.9 & \textbf{36.1} & 20.0 \\
    Davies--Bouldin ($\downarrow$)             & 1.463 & 1.346 & \textbf{1.140} & 1.221 \\
    Alignment ($\downarrow$)                   & 0.510 & 0.448 & \textbf{0.142} & 0.277 \\
    Uniformity                                 & $-3.293$ & $-3.386$ & $-2.729$ & $-3.169$ \\
    \midrule
    \multicolumn{5}{l}{\emph{Cross-modal (metric)}}\\
    cos--Tanimoto (mass null $0.370$) & 0.626 & \textbf{0.661} & 0.548 & -- \\
    Mutual-$k$NN $\to$ MoLFormer, $k{=}10$ ($\uparrow$) & \textbf{0.091} & 0.090 & 0.033 & 0.037 \\
    Mutual-$k$NN $\to$ ChemBERTa, $k{=}10$ ($\uparrow$) & \textbf{0.061} & 0.060 & 0.024 & 0.027 \\
    Linear CKA $\to$ MoLFormer                 & 0.265 & 0.209 & 0.260 & 0.241 \\
    Linear CKA $\to$ ChemBERTa                 & 0.243 & 0.170 & 0.280 & 0.253 \\
    \midrule
    Frozen retrieval MRR                      & \textbf{0.227} & \textbf{0.227} & 0.121 & 0.170 \\
    \bottomrule
  \end{tabular}
\end{table}

\section{Cosine--Tanimoto correlation}
\label{app:correlation}

\begin{figure}[htbp]
  \centering
  \includegraphics[width=\linewidth]{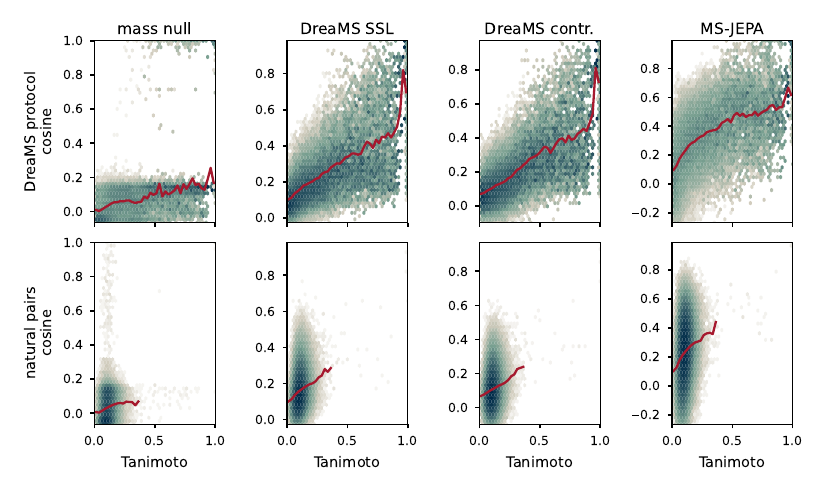}
  \caption{%
    \textbf{Cosine against Morgan Tanimoto} for four representations of the same $3{,}170$ test molecules (one random spectrum each), as a pair density with the mean cosine per Tanimoto bin (line).
    Top: DreaMS's protocol, $40$ Tanimoto bins with $2{,}500$ pairs each drawn with replacement; Pearson $0.370$ (mass null), $0.626$ (DreaMS SSL), $0.661$ (DreaMS contrastive), $0.548$ (MS-JEPA).
    Bottom: $200{,}000$ uniformly drawn molecule pairs; Pearson $0.189$, $0.228$, $0.185$, $0.200$ in the same order.%
  }
  \label{fig:cos_tanimoto}
\end{figure}

\textbf{Mass null.}
DreaMS's fixed precursor Fourier features define a parameter-free $11{,}994$-dimensional representation of precursor $m/z$ alone.
Its cosine similarity closely tracks log mass difference ($r=-0.916$).
On the same molecule pairs as the learned encoders, its cosine--Tanimoto correlation is $0.370$ under the DreaMS protocol ($56\%$ of the contrastive encoder's $0.661$) and $0.189$ under uniform sampling, where the encoders reach $0.185$--$0.228$.

\textbf{Sampling.}
The DreaMS protocol uses Morgan fingerprints (radius $2$, $2{,}048$ bits), one random spectrum per SMILES, and $2{,}500$ pairs sampled with replacement from each of $40$ Tanimoto bins (width $0.025$).
This raises mean Tanimoto from $0.108$ to $0.501$ and all correlations, including the mass null, by $2.0$--$3.6\times$.
On this test fold, $28$ bins contain fewer unique pairs than the quota, requiring repeated pairs at high similarity.
Uniformly sampled pairs concentrate below Tanimoto $0.3$ (Fig.~\ref{fig:cos_tanimoto}).

\section{Extended comparison between encoders}
\label{app:molenc}

Figure~\ref{fig:kxk} scores every pair among five molecule and four spectrum representations: a Morgan fingerprint (radius 2, $4{,}096$ bits), MoLFormer \citep{ross_large-scale_2022}, ChemBERTa \citep{chithrananda_chemberta_2020}, SMI-TED \citep{soares_smited_2024} ($289$M parameters) and MiniMol \citep{klaser_minimol_2024} ($10$M, graph-based) on the molecule side; DreaMS SSL, DreaMS contrastive, MS-JEPA and spectral entropy similarity \citep{li_spectral_2021} on the spectrum side.
Spectrum encoders enter as one normalized centroid per molecule.
The entropy kernel has no embedding, so its molecule-level entry is the most similar pair of spectra between two molecules, and its CKA uses that similarity matrix directly.

\begin{figure}[htbp]
  \centering
  \includegraphics[width=0.92\linewidth]{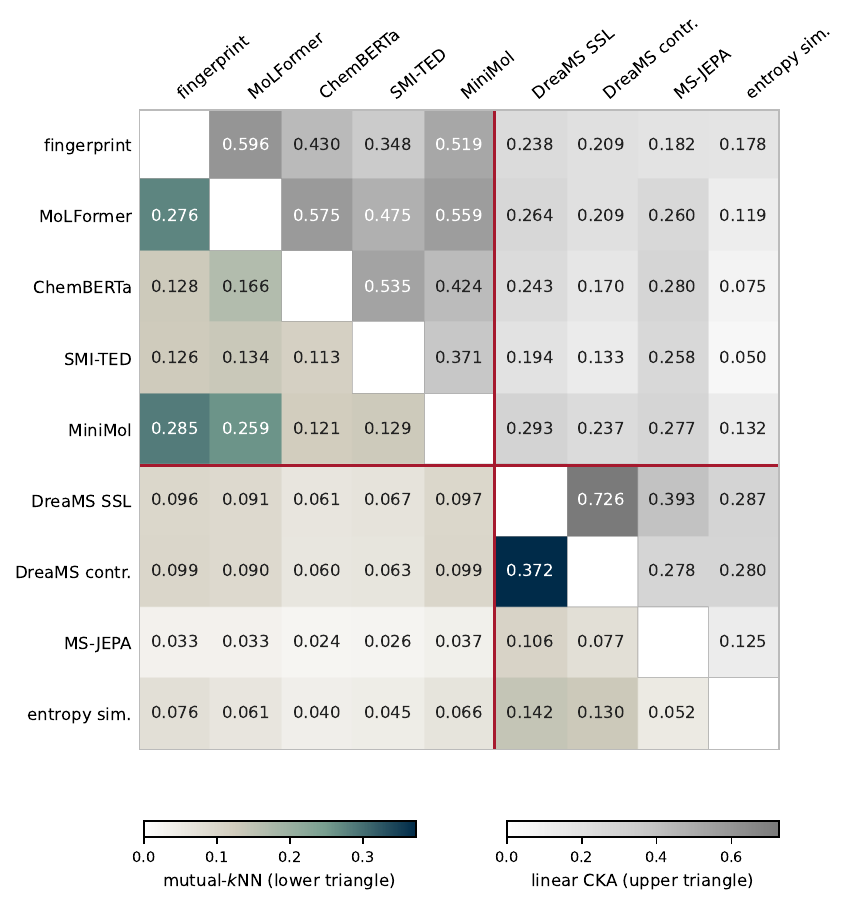}
  \caption{%
    \textbf{Representational alignment across nine representations} on the $2{,}994$ MassSpecGym test molecules covered by all of them.
    Lower triangle: mutual-$k$NN at $k{=}10$ (chance $0.003$); upper triangle: linear CKA.
    Crimson lines separate molecule (top left) from spectrum (bottom right) representations; each triangle is colored on its own scale.%
  }
  \label{fig:kxk}
\end{figure}

Cross-modal mutual-$k$NN is at most $0.099$ across all twenty molecule--spectrum pairs, compared with $0.11$--$0.29$ among molecule representations and $0.37$ between DreaMS stages.
Fingerprints and MiniMol align best with DreaMS ($0.096$--$0.099$), followed by MoLFormer ($0.090$--$0.091$), SMI-TED ($0.063$--$0.067$), and ChemBERTa ($0.060$--$0.061$).
Fingerprints and MiniMol also agree most strongly with each other ($0.285$).

Without learning, entropy similarity reaches $0.076$ against fingerprints: $79\%$ of DreaMS SSL's $0.096$, and above its agreement with ChemBERTa ($0.040$) or SMI-TED ($0.045$).
MS-JEPA agrees least with every molecule representation ($0.024$--$0.037$) and with entropy similarity ($0.052$, versus DreaMS's $0.130$--$0.142$).
CKA gives a different ordering: ChemBERTa--SMI-TED scores $0.535$ despite mutual-$k$NN of $0.113$, and MS-JEPA exceeds DreaMS contrastive against four of five molecule representations.

\section{The molecule-tower term}
\label{app:msjepa}

We adapt VA-VAE's foundation-alignment loss \citep{yao_vavae_2025}: cosine registration to MoLFormer through a learned linear projection (\emph{mcos}), plus matching of within-batch pairwise similarities (\emph{mdms}).
Total weight is $1$, mdms weight $0.25$, and margins $0$.
Each intervention has a matched control.
Encoders are then frozen and evaluated with the in-batch MoLFormer adapter; test-molecule differences are paired.

\textbf{MS-JEPA.}
The term acts on $h_{\mathrm{mol}}$.
Control, mcos, and mcos+mdms use one GPU, batch $64$, $160$ epochs ($40{,}800$ steps), and final checkpoints.
The control's frozen MRR is $0.105$, below the main encoder's $0.121$.
MRR rises $0.105\to0.150\to0.192$ and R@1 $0.038\to0.065\to0.091$: mcos adds $0.045$ [$0.026$, $0.069$], and mdms adds $0.042$ [$0.025$, $0.060$].
Mutual-$k$NN rises $0.029\to0.052\to0.064$; silhouette changes $0.487\to0.290\to0.385$.

\textbf{DreaMS.}
The term acts on the contrastive stage's pre-head precursor token through a $1{,}024\to768$ projection.
Frozen retrieval changes little: MRR $0.227\to0.230$ (difference $+0.003$ [$-0.012$, $0.017$]), R@1 $0.124\to0.119$, and mutual-$k$NN $0.088\to0.092$.
Same-molecule alignment changes $0.437\to0.044$, CKA $0.203\to0.455$, and silhouette $0.298\to0.157$.

\section{Candidate-aware objectives}
\label{app:leak}

The two objectives differ only in the negatives.
With a per-GPU batch $\{(s_i,m_i)\}_{i=1}^n$, $n=64$, and $N_i=\{i\}\cup\{j:m_j\neq m_i\}$ (molecules compared by the first InChIKey block),
\begin{equation}
  \mathcal{L}_{\mathrm{batch}}=-\frac1n\sum_{i}\log
  \frac{\exp\rho_\theta(s_i,m_i)}{\sum_{j\in N_i}\exp\rho_\theta(s_i,m_j)},
  \quad
  \mathcal{L}_{\mathrm{cand}}=-\E_{(s,m^\star)}\Big[
  \log \tfrac{\exp \rho_\theta(s,m^\star)}{\sum_{m'\in\Cand_{128}(s)}\exp \rho_\theta(s,m')}\Big],
  \label{eq:cand}
\end{equation}
where $\Cand_{128}(s)$ is $m^\star$ and $127$ molecules drawn uniformly without replacement from the query's official list $\Cand(s)$ (at most $256$), redrawn every step; a list with fewer than $128$ entries ($0.8\%$ of training queries) is used whole.

\textbf{Candidate construction.}
Mass-based lists are conditioned on a window around the true molecule's mass, using the same procedure in training and evaluation.
The concern here is the transferable difference between library structures and decoys.

\textbf{The optimum.}
With $K$ negatives drawn i.i.d.\ from a proposal $r$ and the positive from $p(\cdot\mid s)$, the population minimizer of the InfoNCE risk is the log density ratio \citep{oord_cpc_2018,poole_variational_2019}
\begin{equation}
  \rho^\star_r(s,m) = \log \frac{p(m\mid s)}{r(m)} + c(s)
  = \underbrace{\log\frac{p(m)}{r(m)}}_{A_r(m)} \;+\;
  \underbrace{\log\frac{p(s\mid m)}{p(s)}}_{\mathrm{PMI}(s;m)} \;+\; c(s).
  \label{eq:opt2}
\end{equation}
For independent in-batch negatives $r=p$, the prior term vanishes.
For an independent candidate proposal $r=q\ne p$, it remains.
This idealization explains the incentive to learn a library--decoy prior.
Actual benchmark negatives are query-dependent and sampled without replacement.
Property-matched decoys \citep{gupta_confronting_2026} reduce differences on the matched properties.

\textbf{Blind controls.}
The primary control keeps the trained adapters and ranks each official list by $\bar\rho(m)=\E_s[\rho_\theta(s,m)]$ over the $194{,}119$ training spectra, each with its own conditions.
It uses neither query spectra nor query conditions.
For global scores, $\bar\rho(m)=\tau^{-1}\langle\E_s[z_s],z_m\rangle$ exactly; this differs from adapting the mean encoder output because the adapter is nonlinear.
For MaxSim, we estimate the expectation with $16{,}384$ training spectra sampled without replacement.

The secondary control (global scores only) replaces the encoder output by its training mean but retains query adduct and collision energy at the adapter.
We report $\Delta_0=\mathrm{MRR}_{\mathrm{control}}-0.024$ and $\Delta_s=\mathrm{MRR}_{\mathrm{full}}-\mathrm{MRR}_{\mathrm{control}}$.

\begin{table}[h]
  \centering\footnotesize
  \caption{%
    Spectrum-blind controls, frozen DreaMS contrastive, MassSpecGym test fold.
    Expected score: primary control (Table~\ref{tab:results}, left); mean embedding: secondary control, conditions kept.%
  }
  \label{tab:blind}
  \setlength{\tabcolsep}{4pt}
  \begin{tabular}{llccccccc}
    \toprule
    & & & \multicolumn{3}{c}{expected score} & \multicolumn{3}{c}{mean embedding} \\
    \cmidrule(lr){4-6}\cmidrule(lr){7-9}
    Tower & Negatives & MRR & blind & $\Delta_0$ & $\Delta_s$ & blind & $\Delta_0$ & $\Delta_s$ \\
    \midrule
    ChemBERTa & in-batch  & 0.147 & 0.020 & $-0.004$ & 0.127 & 0.039 & $+0.015$ & 0.108 \\
    ChemBERTa & candidate & 0.211 & 0.147 & $+0.122$ & 0.064 & 0.155 & $+0.131$ & 0.056 \\
    MoLFormer & in-batch  & 0.227 & 0.018 & $-0.006$ & 0.209 & 0.034 & $+0.009$ & 0.193 \\
    MoLFormer & candidate & 0.303 & 0.169 & $+0.144$ & 0.135 & 0.160 & $+0.136$ & 0.143 \\
    \bottomrule
  \end{tabular}
\end{table}

At fixed towers (Table~\ref{tab:blind}), candidate training raises expected-score blind MRR by $0.151$ [$0.125$, $0.178$] for MoLFormer and $0.126$ [$0.105$, $0.148$] for ChemBERTa, versus full-MRR gains of $0.077$ [$0.059$, $0.095$] and $0.064$ [$0.049$, $0.078$].
Thus $\Delta_s$ falls by $0.074$ [$0.050$, $0.097$] and $0.063$ [$0.042$, $0.082$].
The blind gain accounts for $0.52$ [$0.45$, $0.59$] and $0.66$ [$0.56$, $0.75$] of full MRR above random.

The mean-embedding control gives similar candidate-training gains ($0.127$ [$0.107$, $0.148$] and $0.115$ [$0.094$, $0.136$]).
Its in-batch scores are slightly above chance ($\Delta_0=0.009$ [$0.007$, $0.012$] and $0.015$ [$0.011$, $0.020$]), whereas expected-score controls are not ($-0.006$ [$-0.010$, $0.000$] and $-0.004$ [$-0.007$, $0.001$]).

\subsection{Other candidate-trained models}
\label{app:candtrained}

JESTR \citep{kalia_jestr_2025} trains a spectrum MLP and a molecular-graph GCN from scratch with in-batch InfoNCE in both directions, $\mathcal{L}_{\mathrm{NCE}}$, on cosine scores.
For the last $3\%$ of its epochs it adds a penalty on $L$ entries of each training query's mass candidate list, sorted by Tanimoto similarity to the query and visited in turn; JESTR-NR omits it.
With unit spectrum embeddings $\bar s_i$ of a batch of $n$ queries and unit embeddings $\bar z_{ij}$ of their listed decoys,
\begin{equation}
  \mathcal{L}=(1-w)\,\mathcal{L}_{\mathrm{NCE}}
  +\frac{w}{nL}\sum_{i=1}^{n}\sum_{j=1}^{L}\langle\bar s_i,\bar z_{ij}\rangle,
  \qquad w=0.1 .
  \label{eq:jestr}
\end{equation}
We train both on our split with the official mass lists and the evaluation of App.~\ref{app:eval}.
The blind control is the expected score, $\langle\mu,\bar z_m\rangle$ with $\mu=\E_s[\bar s]$ over the training spectra (Table~\ref{tab:jestr}).

\textbf{Published recipe.}
We use the public architecture and hyperparameters: spectrum dropout $0.4$; GCN with BatchNorm, dropout $0.2$, and max pooling; $\tau=0.05$; Adam at $5\times10^{-4}$; batch $32$; and $L=32$ decoys in a separate train-mode forward pass.
Our deviations are all-adduct training for $1{,}000$ epochs, one random spectrum per query per epoch, masking same-molecule negatives, removing the query from its candidate list, and using the last checkpoint.
The public code restricts adducts to $[\mathrm{M}{+}\mathrm{H}]^+$ and stops early on training loss.
Both arms share the first $970$ epochs.

The penalty changes full MRR by $-0.012$ [$-0.017$, $-0.007$] and blind MRR by $-0.033$ [$-0.047$, $-0.020$].
Nine repeats of the last $30$ epochs from a second run give $-0.004\pm0.001$ full and $+0.005\pm0.002$ blind MRR (mean $\pm$ s.e.); three further runs give blind changes of $+0.011$, $-0.007$, and $-0.004$.
The effect is small and inconsistent.
Without the penalty, JESTR-NR has blind MRR $0.129$ ($\Delta_0=0.105$), but $94\%$ of this gain comes from $153$ molecules with at least $20$ spectra ($39\%$ of test spectra).
Molecule-averaged blind MRR is $0.043$ (random $0.024$).
Mean normalized spectrum embeddings have norms $0.14$ and $0.13$; using mean unnormalized embeddings instead gives blind MRR $0.037$ and $0.036$.

\textbf{Simplified recipe.}
We retain Eq.~\eqref{eq:jestr} but use a $1000$--$512$--$512$--$512$ spectrum MLP and a three-layer, $512$-wide GCN with mean pooling, neither with dropout nor BatchNorm.
Training uses unmasked InfoNCE, $\tau=0.07$, AdamW at $10^{-3}$, weight decay $0.01$, and $60$ epochs over all spectra.
The penalty acts in the last two epochs on $L=5$ of the $20$ candidates most Tanimoto-similar to the query.
JESTR-NR is at chance under the blind control.
The penalty improves blind MRR in all five seeds, adding $0.056\pm0.016$ blind MRR ($0.024\pm0.003$ per molecule) and $0.013\pm0.002$ full MRR.

At penalty onset, its gradient-norm ratio to $(1-w)\mathcal L_{\mathrm{NCE}}$ is $0.0059$ in the published recipe versus $0.073$ in the simplified recipe.
This measures relative gradient strength before Adam's adaptive scaling.
Raising $w$ to $0.58$ restores the ratio; encoding decoys with positives in the same BatchNorm batch then gives $+0.068$ blind and $+0.020$ full MRR in three of three seeds.
Separate batches allow BatchNorm statistics to distinguish positives from decoys and cause collapse in two of three seeds.

\textbf{Why the penalty moves the blind scorer.}
Write $\bar s_i=\mu+\varepsilon_i$.
The penalty splits into the listed decoys' blind scores and a spectrum-specific remainder,
\begin{equation*}
  \frac{w}{nL}\sum_{i,j}\langle\bar s_i,\bar z_{ij}\rangle
  =\frac{w}{nL}\sum_{i,j}\langle\mu,\bar z_{ij}\rangle
  +\frac{w}{nL}\sum_{i,j}\langle\varepsilon_i,\bar z_{ij}\rangle ,
\end{equation*}
so the negative gradient with respect to a decoy's direction is proportional to $-(I-\bar z_{ij}\bar z_{ij}^{\top})(\mu+\varepsilon_i)$: it points away from the listing spectrum, with a shared component along $-\mu$.
Candidate InfoNCE has the same directional contribution, weighted by $p_\theta(m'\mid s)/\tau$ instead of $w/(nL)$.
Under the independent-proposal idealization, its expected score equals $A_q(m')+\E_s[\mathrm{PMI}(s;m')]$ up to a constant (Eq.~\eqref{eq:opt2}); the blind control therefore need not isolate $A_q$ alone.

A test-set blind gain requires the molecule encoder to transfer this effect more strongly to decoys than to positives.
In units of $\langle\mu,\bar z\rangle/\lVert\mu\rVert$, the simplified GCN lowers positive and decoy scores by $0.089$ and $0.168$ on training queries, and by $0.086$ and $0.165$ on test queries (gaps $0.079$ and $0.080$).
The published GCN in the otherwise simplified recipe produces gaps of $0.024$ on training and $0.001$ on test, indicating little transfer of this separation.

\begin{table}[h]
  \centering\footnotesize
  \caption{%
    Effect of JESTR's candidate-list penalty, MassSpecGym test fold.
    Change: JESTR minus JESTR-NR, averaged over paired runs (the published runs share their first $970$ epochs).
    Blind: expected-score control (random $0.024$).
    Reweighted: published recipe with $w=0.58$ and the decoys in the positives' BatchNorm batch.%
  }
  \label{tab:jestr}
  \begin{tabular}{lccccc}
    \toprule
    & & \multicolumn{2}{c}{JESTR-NR} & \multicolumn{2}{c}{Change with penalty} \\
    \cmidrule(lr){3-4}\cmidrule(lr){5-6}
    Recipe & Runs & MRR & Blind & MRR & Blind \\
    \midrule
    Published   & 9 & 0.230 & 0.118 & $-0.004$ & $+0.005$ \\
    Simplified  & 5 & 0.163 & 0.029 & $+0.013$ & $+0.056$ \\
    Reweighted  & 3 & 0.235 & 0.118 & $+0.020$ & $+0.068$ \\
    \bottomrule
  \end{tabular}
\end{table}

\section{Improved baselines}
\label{app:baselines}

Unless stated otherwise, all arms use frozen MoLFormer, in-batch negatives, $100$ molecule epochs, validation-selected checkpoints, and seed $42$ (App.~\ref{app:eval}).
Table~\ref{tab:baselines} reports results.

\textbf{Initialization and learning rates.}
DreaMS contrastive is unfrozen from the first step at $5\times10^{-6}$, $1.5\times10^{-5}$, or $5\times10^{-5}$ ($104.0$M trainable parameters, versus $8.4$M frozen).
MS-JEPA (A) and (B) use the same rates.
Validation selects (A) at $5\times10^{-5}$ (MRR $0.244$); (B) at that rate has higher test MRR by $0.014$ [0.001, 0.027].
For fresh encoders, validation selects $10^{-4}$ over $3\times10^{-4}$ ($0.235$ versus $0.230$).
These selected rates are retained throughout Table~\ref{tab:results}.

\textbf{Ablations.}
\emph{MaxSim} uses the scoring rule in App.~\ref{app:methods:msmoljepa}.
\emph{Simulated pretraining} uses $46{,}000$ steps of $256$ simulated spectra, the shared optimizer, and warm-up and cosine decay over those steps.
Real-data validation runs every $2{,}000$ steps.
The best checkpoint initializes the encoder, adapters, and temperature for real-data fine-tuning, with a fresh optimizer and schedule.
Global and MaxSim scoring have separate pretraining runs.
\emph{Candidate training} replaces in-batch negatives during fine-tuning from the MaxSim pretraining checkpoint (Eq.~\eqref{eq:cand}).

Table~\ref{tab:baselines}(b) reports successive changes.
After simulated pretraining, pretrained MS-JEPA exceeds the fresh encoder by $0.012$ [0.001, 0.024] MRR with MaxSim but trails by $0.024$ [0.001, 0.049] with candidate training.
Candidate training raises blind MRR from $0.013$ to $0.083$ (fresh) and $0.098$ (pretrained).

\textbf{Training length.}
Fresh encoders also run for $48$ and $200$ molecule epochs ($4{,}656$ and $19{,}400$ steps), with global and MaxSim scoring.
In all real-only end-to-end runs, training loss falls from $\ln64\approx4.2$ to $0.02$--$0.07$ (last $500$ steps), while validation peaks before training ends.
The $200$-epoch runs peak at steps $4{,}800$ (global, epoch $49$) and $11{,}640$ (MaxSim, epoch $120$), then lose $0.031$ and $0.020$ validation MRR.
Best validation MRR still rises with training length: $0.215$, $0.235$, $0.253$ globally and $0.234$, $0.240$, $0.248$ with MaxSim at $48$, $100$, $200$ epochs.
Test MRR shows no clear gain beyond $100$ epochs.
After simulated pretraining, fresh MaxSim validation peaks at the last step.

\begin{table}[h]
  \centering\scriptsize
  \caption{%
    \textbf{Improved baselines}, MassSpecGym test fold, one seed each.
    $\Delta$: paired MRR difference with a $95\%$ interval from $1{,}000$ bootstrap draws over test molecules (InChIKey-14).
    (a) Arms trained on real spectra; $\Delta$ against the reference named in each block.
    (b) One-change steps of Table~\ref{tab:results} (right) and pretraining on simulated spectra alone (sim only), for the fresh encoder and the pretrained MS-JEPA (A) encoder.%
  }
  \label{tab:baselines}
  \setlength{\tabcolsep}{4pt}
  \begin{tabular}{lccc}
    \multicolumn{4}{l}{(a)}\\
    \toprule
    Arm & MRR & R@1 & $\Delta$MRR [95\% CI] \\
    \midrule
    \multicolumn{4}{l}{\emph{DreaMS contrastive, backbone LR; $\Delta$ vs.\ frozen}} \\
    frozen & 0.227 & 0.122 & \\
    $5\times10^{-6}$ & 0.255 & 0.144 & $+0.028$ [0.016, 0.041] \\
    $1.5\times10^{-5}$ & 0.255 & 0.136 & $+0.028$ [0.016, 0.041] \\
    $5\times10^{-5}$ & 0.244 & 0.130 & $+0.018$ [0.005, 0.029] \\
    \multicolumn{4}{l}{\emph{MS-JEPA (A), aggregated, encoder LR; $\Delta$ vs.\ frozen}} \\
    frozen & 0.121 & 0.047 & \\
    $5\times10^{-6}$ & 0.210 & 0.107 & $+0.089$ [0.075, 0.104] \\
    $1.5\times10^{-5}$ & 0.236 & 0.124 & $+0.115$ [0.098, 0.132] \\
    $5\times10^{-5}$ (ladder) & 0.243 & 0.128 & $+0.121$ [0.105, 0.140] \\
    \multicolumn{4}{l}{\emph{MS-JEPA (B), [CLS], encoder LR; $\Delta$ vs.\ frozen}} \\
    frozen & 0.170 & 0.077 & \\
    $5\times10^{-6}$ & 0.219 & 0.116 & $+0.048$ [0.038, 0.059] \\
    $1.5\times10^{-5}$ & 0.235 & 0.128 & $+0.065$ [0.053, 0.078] \\
    $5\times10^{-5}$ & 0.257 & 0.143 & $+0.087$ [0.070, 0.104] \\
    \multicolumn{4}{l}{\emph{Fresh encoder, encoder LR; $\Delta$ vs.\ MS-JEPA (A) at $5\times10^{-5}$}} \\
    $1\times10^{-4}$ (ladder) & 0.232 & 0.124 & $-0.011$ [$-0.025$, 0.003] \\
    $3\times10^{-4}$ & 0.212 & 0.111 & $-0.031$ [$-0.045$, $-0.017$] \\
    \multicolumn{4}{l}{\emph{Fresh encoder, training length; $\Delta$ vs.\ 100 epochs (0.232 global, 0.241 MaxSim)}} \\
    global, 48 epochs & 0.216 & 0.113 & $-0.016$ [$-0.029$, $-0.003$] \\
    global, 200 epochs & 0.237 & 0.121 & $+0.005$ [$-0.011$, 0.020] \\
    MaxSim, 48 epochs & 0.244 & 0.126 & $+0.002$ [$-0.009$, 0.014] \\
    MaxSim, 200 epochs & 0.239 & 0.135 & $-0.003$ [$-0.018$, 0.012] \\
    \bottomrule
  \end{tabular}

  \vspace{6pt}
  \begin{tabular}{lcc}
    \multicolumn{3}{l}{(b)}\\
    \toprule
    Step & fresh & pretrained \\
    \midrule
    \multicolumn{3}{l}{\emph{$\Delta$MRR [95\% CI]}} \\
    real: global $\to$ MaxSim & $+0.010$ [$-0.008$, 0.027] & $-0.009$ [$-0.024$, 0.006] \\
    global: real $\to$ sim$\to$real & $+0.137$ [0.117, 0.157] & $+0.131$ [0.112, 0.151] \\
    sim$\to$real: global $\to$ MaxSim & $+0.094$ [0.075, 0.116] & $+0.102$ [0.081, 0.123] \\
    sim$\to$real, MaxSim: in-batch $\to$ candidate & $+0.103$ [0.080, 0.128] & $+0.067$ [0.044, 0.088] \\
    \quad same step, spectrum-blind MRR & $+0.070$ [0.057, 0.085] & $+0.085$ [0.070, 0.102] \\
    global: real $\to$ sim only & $+0.005$ [$-0.019$, 0.029] & $+0.017$ [$-0.006$, 0.038] \\
    MaxSim: real $\to$ sim only & $+0.105$ [0.077, 0.134] & $+0.110$ [0.082, 0.141] \\
    \midrule
    \multicolumn{3}{l}{\emph{sim only, MRR / R@1}} \\
    global & 0.236 / 0.140 & 0.260 / 0.158 \\
    MaxSim & 0.347 / 0.251 & 0.344 / 0.249 \\
    \bottomrule
  \end{tabular}
\end{table}

\section{Simulated spectra}
\label{app:sim}

The pretraining corpus of \S\ref{sec:baselines} is built in three steps: collect molecules and remove those near the test fold, simulate their spectra, and remove every molecule that the benchmark's retrieval ranks.

\textbf{Molecules.}
We pool MOSES \citep{polykovskiy_moses_2020}, COCONUT \citep{sorokina_coconut_2021}, DSSTox \citep{grulke_dsstox_2019}, HMDB \citep{wishart_hmdb_2022}, MassSpecGym's $3.92$M PubChem structures \citep{bushuiev_massspecgym_2024}, and the $2.50$M pretraining structures of \citet{liu_massspecgym_2026}: $7{,}627{,}340$ distinct canonical isomeric SMILES.
We remove, in order, structures that
\begin{enumerate}\itemsep1pt
  \item share 2D connectivity with any MassSpecGym molecule ($65{,}410$);
  \item have the Morgan or RDKit fingerprint of a test molecule ($1{,}734$);
  \item cannot form [M+H]$^+$ or [M+Na]$^+$, or are unstable on the LC time scale ($350{,}568$), by rules that exclude no MassSpecGym molecule;
  \item are not proven to lie at MCES distance $\ge10$, MassSpecGym's threshold, from every test molecule ($265{,}749$), using lower bounds or exact myopic MCES \citep{kretschmer_mces_2025}.
\end{enumerate}
The $6{,}943{,}879$ remaining structures and the $28{,}929$ MassSpecGym molecules form $6.19$M molecules (InChIKey first blocks).

\textbf{Simulators.}
GLACIER \citep{wang_glacier_2026} is the authors' non-contrastive checkpoint, reported as trained on the official MassSpecGym split.
FraGNNet \citep{young_fragnnet_2025} is retrained on the training-fold spectra.
Both models simulate [M+H]$^+$ spectra at eight normalized collision energies, the medians of equal-count bins of training [M+H]$^+$ NCE.
Two simulators with different errors act as domain randomization \citep{tobin_domain_2017}.

\textbf{Removing retrieval candidates.}
MassSpecGym molecules keep their folds; the others are split $98/1/1$ by hash, and only the training split is used ($6{,}064{,}935$ molecules).
From it we drop every validation or test query and every molecule in any of their mass- or formula-based candidate lists: $187{,}223$ molecules and $3.48$M spectra, $1{,}798$ of them MassSpecGym training molecules.
The remaining $5{,}877{,}712$ molecules carry $98.2$M spectra ($52.1$M GLACIER, $46.1$M FraGNNet).

\textbf{Sim-to-real differences.}
Three mismatches are corrected.
(i)~GLACIER predicts square-root intensities and its spectra are too flat; we square them and drop peaks below $6.1\times10^{-4}$ of the base peak, a floor fit to real peak counts on energy-matched validation pairs.
Spectra then average $27.9$ (GLACIER) and $29.9$ (FraGNNet) peaks, against $31.5$ measured.
(ii)~The pooled molecules are heavier than MassSpecGym's; reweighting by $25$~Da precursor bins lowers the 1-Wasserstein distance to measured precursor masses from $38.4$ (GLACIER) and $34.7$~Da (FraGNNet) to $0.94$~Da, with bins absent from the measured data, almost all above $1{,}000$~Da, given zero weight.
Each simulator receives half the probability.
(iii)~Simulated $m/z$ values are exact, whereas measured ones carry ppm-scale error; the augmentation shared with measured spectra (peak dropout $0.2$, log-normal intensity noise $\sigma=0.3$, $5$~ppm $m/z$ jitter) blurs them.
Other differences remain: one adduct, one instrument and eight energies, where measured spectra also include [M+Na]$^+$ and QTOF, and no noise or contaminant peaks.

\textbf{Before fine-tuning.}
Without measured training spectra (measured validation spectra select checkpoints), global scoring matches real-only MRR (Table~\ref{tab:baselines}(b)), with higher R@1 ($+0.016$ [$-0.005$, $0.035$] fresh, $+0.029$ [$0.007$, $0.050$] pretrained) but lower R@20 ($-0.067$ [$-0.113$, $-0.022$], $-0.069$ [$-0.111$, $-0.029$]).
MaxSim reaches $0.347$ (fresh) and $0.344$ (pretrained) MRR, $0.105$ and $0.110$ above real-only MaxSim training.
Pretrained initialization helps global scoring ($+0.023$ [$0.016$, $0.031$]) but not MaxSim ($-0.002$ [$-0.008$, $0.002$]).
Fine-tuning on measured spectra adds $0.114$--$0.132$ MRR.

\end{document}